\documentclass{article}

\usepackage[preprint]{neurips_2026}

\usepackage[utf8]{inputenc}
\usepackage[T1]{fontenc}
\usepackage[hidelinks]{hyperref}
\usepackage{url}
\usepackage{booktabs}
\usepackage{amsfonts}
\usepackage{amsmath,amssymb,amsthm,mathtools}
\usepackage{microtype}
\usepackage{xcolor}
\usepackage{graphicx}

\title{Separating Shortcut Transition from Cross-Family OOD Failure in a Minimal Model}

\author{%
Hongmin Li\textsuperscript{1,2}\\[0.25em]
\normalfont\textsuperscript{1}School of Life Science and Technology, Institute of Science Tokyo\\
\normalfont 2-12-1 Ookayama, Meguro-ku, Tokyo 152-8550, Japan\\
\normalfont\textsuperscript{2}Department of Computational Biology and Medical Sciences\\
\normalfont Graduate School of Frontier Sciences, The University of Tokyo\\
\normalfont 5-1-5 Kashiwanoha, Kashiwa-shi, Chiba 277-8561, Japan\\
\normalfont\texttt{lihongmin@edu.k.u-tokyo.ac.jp}\\[0.25em]
\normalfont\small Researcher, School of Life Science and Technology, Institute of Science Tokyo;\\
\normalfont\small Guest Researcher, Department of Computational Biology and Medical Sciences,\\
\normalfont\small Graduate School of Frontier Sciences, The University of Tokyo.\\
\normalfont\small ORCID: \href{https://orcid.org/0000-0003-0228-0600}{0000-0003-0228-0600}
}

\newtheorem{theorem}{Theorem}

\newcommand{\E}{\mathbb{E}}
\newcommand{\Prob}{\mathbb{P}}
\newcommand{\calE}{\mathcal{E}}
\newcommand{\train}{\mathrm{train}}
\newcommand{\test}{\mathrm{test}}
\newcommand{\one}{\mathbf{1}}
\newcommand{\Unif}{\mathrm{Unif}}
\newcommand{\risk}{R}
\newcommand{\loss}{L}
\newcommand{\rhoavg}{\bar\rho_{\train}}

\DeclareMathOperator{\sign}{sign}
\DeclareMathOperator*{\argmin}{arg\,min}

\begin{document}

\maketitle

\begin{abstract}
Shortcut features are often invoked to explain out-of-distribution (OOD)
failure, but training correlation, learned shortcut use, and test-time failure
need not coincide. We study a minimal binary model with one invariant
coordinate and one family-dependent shortcut coordinate. In the deterministic
regime, positive average shortcut correlation pulls logistic ERM toward
positive shortcut weight, but ridge regularization keeps the classifier
invariant-dominated and prevents deterministic OOD failure. When the invariant
coordinate is noisy, ridge-logistic ERM switches to the shortcut rule once the
training shortcut signal exceeds the invariant signal. Whether that transition
causes failure depends on the held-out family: weaker shortcut
correlation yields positive excess risk, and sign-flipped families yield
above-chance error. Synthetic checks match these analytic regimes and show that
the same training-side transition can have different held-out consequences. The
model separates shortcut attraction, shortcut-rule transition, and
cross-family OOD failure.

\end{abstract}

\section{Introduction}
Training families can contain predictive but unstable shortcuts: features that
help on seen families but can induce failure once label correlations weaken or
flip. This tension underlies shortcut learning and invariance-based OOD work
\citep{geirhos2020shortcut,peters2016causal,rojascarulla2018invariant,
scholkopf2021toward}.

The surrounding literature spans several related questions. Domain adaptation
and representation-alignment methods emphasize source-target mismatch
\citep{bendavid2010theory,ganin2016domain,sun2016deepcoral}, causal and
invariant-prediction approaches ask which mechanisms remain stable across
environments \citep{peters2016causal,rojascarulla2018invariant,
arjovsky2019invariant,scholkopf2021toward}, and recent empirical or
theoretical critiques show that robustness- or invariance-motivated objectives
can still absorb shortcuts or fail under group shift
\citep{gulrajani2021search,krueger2021out,kamath2021does,rosenfeld2021risks,
sagawa2020distributionally,koh2021wilds,nam2020learning}.
We focus on a narrower question: which train-side observations already indicate
cross-family failure, and which do not. Existing work shows shortcut
absorption, robustness gaps, or failures of invariance-motivated training. Here
we isolate a different distinction inside one closed-form parameterization:
shortcut attraction, a training-side transition to the shortcut rule, and the
additional test-side condition that turns that transition into actual OOD
failure. The aim is not to replace those broader lines of work, but to make
this separation explicit in a minimal two-coordinate model.

\begin{figure}[t]
\centering
\includegraphics[width=\linewidth]{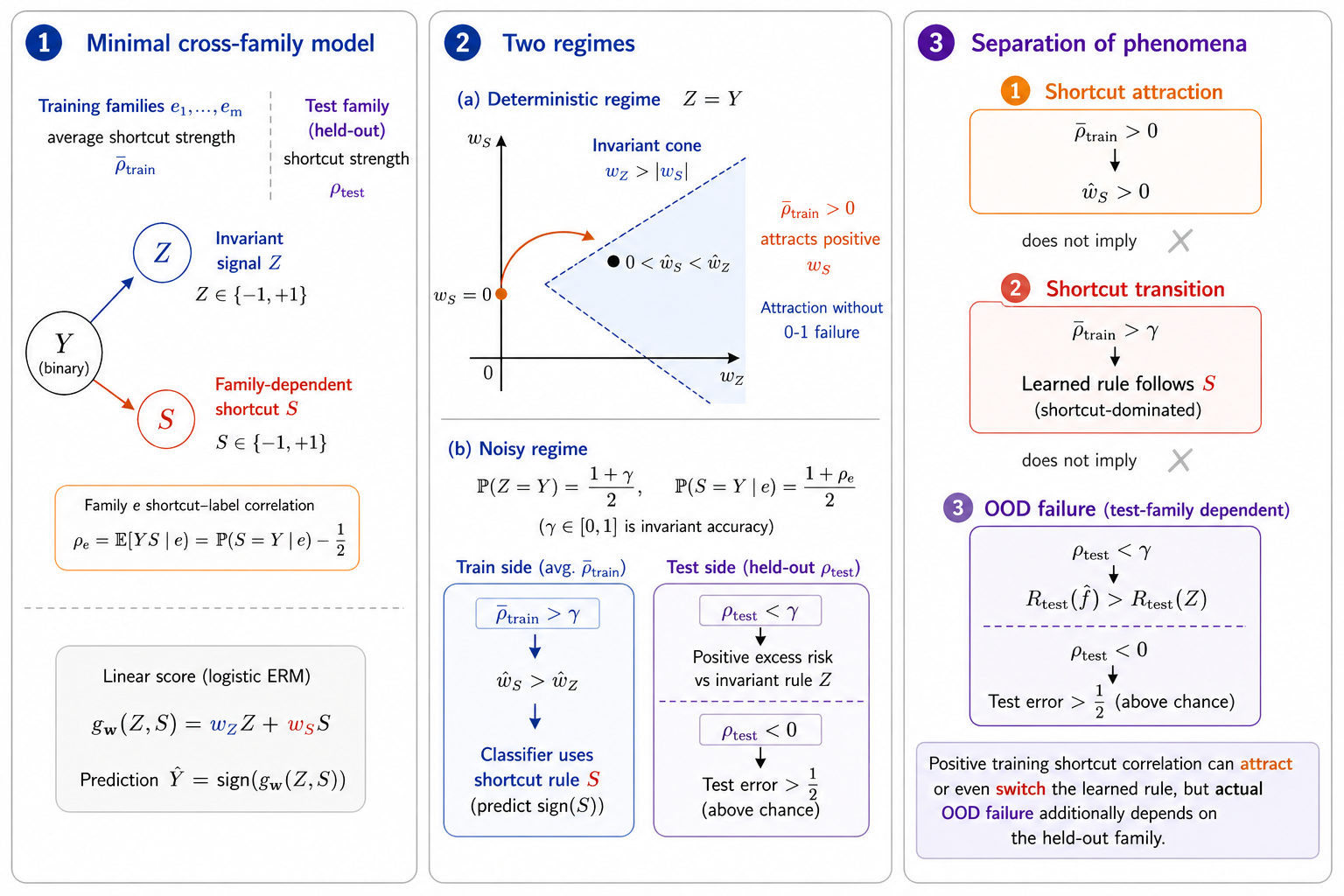}
\caption{Main conceptual separation. The model has an invariant coordinate
$Z$ and a family-dependent shortcut coordinate $S$, with training families
summarized by average shortcut correlation $\rhoavg$ and a held-out test
family by $\rho_{\test}$. Positive $\rhoavg$ can attract surrogate training
toward a positive shortcut weight even when the deterministic classifier
remains invariant-dominated. In the noisy regime, the stronger condition
$\rhoavg>\gamma$ creates a shortcut transition, while actual cross-family OOD
failure additionally depends on the held-out shortcut correlation through
$\rho_{\test}<\gamma$, with $\rho_{\test}<0$ giving above-chance error.}
\label{fig:main-conceptual-overview}
\end{figure}

We study a closed-form binary model with two observed coordinates,
\[
X=(Z,S),
\]
where $Z$ is an invariant signal and $S$ is a family-dependent shortcut. In the
deterministic baseline, the true mechanism is $Y=Z$ and family $e$ is indexed
by shortcut correlation $\rho_e=\E_e[SY]$. Closed-form $0$-$1$ risks and an
explicit train-test gap in the shortcut cone show what positive training
shortcut correlation does and does not imply: logistic ERM is pulled toward
positive shortcut weight, and the test margin worsens when family correlations
weaken or flip, but the ridge solution remains invariant-dominated and avoids
misclassification.

To obtain a direct failure statement, we add invariant noise through
independent agreement variables $A$ and $B_e$ with means $\gamma$ and
$\rho_e$, and set $Z=YA$, $S=YB_e$. In this regime, $\rhoavg$ becomes the
training-side shortcut-transition parameter. The shortcut rule already beats
the invariant rule when $\rhoavg>\gamma$, and ridge-logistic ERM obeys the
same threshold. On test families with $\rho_{\test}<\gamma$, that transition
creates positive excess risk relative to the invariant rule; on sign-flipped
families with $\rho_{\test}<0$, it yields error above chance. The same
training solution can therefore be benign on one held-out family and fail on
another.

Theorems~\ref{thm:ridge-invariant} and~\ref{thm:noisy-ridge-shortcut}
formalize those two steps. The deterministic theorem gives exact geometry,
local shortcut incentive, and test-side degradation without claiming failure.
The noisy theorem gives the shortcut-rule transition and the additional
held-out-family condition for relative or absolute failure. Figure~\ref{fig:main-conceptual-overview}
summarizes this logic, and the appendix supplies full proofs and selector-level
complements.

It does not try to explain shortcut learning in general, and it does not
propose a new robust training objective. Instead, it uses a minimal closed-form
model to show that positive
training shortcut correlation, shortcut-rule selection, and cross-family OOD
failure are distinct statements. The result is a minimal analytic separation
between shortcut attraction, shortcut-rule transition, and failure.

\section{Model and Main Results}
\label{sec:setup}
We work with binary labels $Y \in \{-1,+1\}$ and binary inputs
$X=(Z,S)\in\{-1,+1\}^2$. The coordinate $Z$ is the invariant signal and $S$ is
a family-specific shortcut.

\paragraph{Deterministic family model.}
For each family $e \in \calE$, let
\[
Y \sim \Unif\{-1,+1\}, \qquad Z = Y,
\]
and
\[
S =
\begin{cases}
Y, & \text{with probability } \frac{1+\rho_e}{2},\\[2pt]
-Y, & \text{with probability } \frac{1-\rho_e}{2},
\end{cases}
\]
where $\rho_e \in [-1,1]$ is the shortcut-label correlation of family $e$.

We study linear scores
\[
g_w(Z,S)=w_Z Z + w_S S,
\qquad
f_w(Z,S)=\sign(g_w(Z,S)).
\]
If the learner observes training families $e_1,\dots,e_m$ with weights
$\alpha_1,\dots,\alpha_m$, we write
\[
\rhoavg := \sum_{i=1}^m \alpha_i \rho_{e_i}.
\]

\paragraph{Exact risk geometry.}
Write
\[
a_+ := w_Z + w_S,
\qquad
a_- := w_Z - w_S.
\]
For every family $e$,
\[
\risk_e(w)
=
\frac{1+\rho_e}{2}\,\psi(a_+)
+
\frac{1-\rho_e}{2}\,\psi(a_-),
\]
where $\psi(t)=\one\{t<0\} + \frac12 \one\{t=0\}$. Hence the risk is piecewise
constant on four cones. On the shortcut-dominated cone $w_S>|w_Z|$, the
family-wise error is $(1-\rho_e)/2$, so weaker test shortcut correlation
creates an explicit train-test gap.

Away from the boundaries $a_+=0$ and $a_-=0$, this becomes the explicit
four-cone decomposition
\[
\risk_e(w)=
\begin{cases}
0, & w_Z>|w_S|,\\[2pt]
\frac{1-\rho_e}{2}, & w_S>|w_Z|,\\[6pt]
\frac{1+\rho_e}{2}, & -w_S>|w_Z|,\\[6pt]
1, & w_Z<-|w_S|.
\end{cases}
\]
The nonnegative shortcut cone is the one relevant for the later logistic
analysis.

Averaging over training families simply replaces $\rho_e$ by $\rhoavg$. Thus
on the shortcut-dominated cone with $w_Z,w_S\ge 0$ and $w_S>w_Z$,
\[
\risk_{\train}(w)=\frac{1-\rhoavg}{2},
\qquad
\risk_{\test}(w)=\frac{1-\rho_{\test}}{2},
\]
so
\[
\risk_{\test}(w)-\risk_{\train}(w)=\frac{\rhoavg-\rho_{\test}}{2}.
\]
If the observed families make the shortcut look predictive but the test family
weakens or flips that correlation, the exact cross-family gap is already
positive at the level of $0$-$1$ geometry.

\paragraph{Surrogate shortcut attraction.}
The $0$-$1$ risk is piecewise constant, so it does not reveal the optimization
bias of surrogate training. For logistic loss
$\ell(t)=\log(1+e^{-t})$, the weighted training objective is
\[
\loss_{\train}(w_Z,w_S)
=
\frac{1+\rhoavg}{2}\,\ell(a_+)
+
\frac{1-\rhoavg}{2}\,\ell(a_-).
\]
At $w_S=0$, with $\sigma(t)=(1+e^{-t})^{-1}$,
\[
\left.\frac{\partial \loss_{\train}}{\partial w_S}\right|_{w_S=0}
=
-\rhoavg \,\sigma(-w_Z).
\]
Hence every $w_Z>0$ with $\rhoavg>0$ is locally pushed toward positive shortcut
weight. This deterministic scope isolates \emph{shortcut attraction} without
yet implying OOD failure.

The same formulas also expose the train-test surrogate gap. For a test family
with shortcut correlation $\rho_{\test}$,
\[
\loss_{\test}(w_Z,w_S)-\loss_{\train}(w_Z,w_S)
=
\frac{\rho_{\test}-\rhoavg}{2}\bigl(\ell(a_+)-\ell(a_-)\bigr).
\]
Hence if $w_S>0$ and $\rho_{\test}<\rhoavg$, the positive shortcut weight that
training prefers already worsens the test surrogate.

Reparameterizing by
$u=w_Z+w_S$ and $v=w_Z-w_S$. In the deterministic objective, the $u$ channel
is weighted by $(1+\rhoavg)/2$ and the $v$ channel by $(1-\rhoavg)/2$. Positive
$\rhoavg$ therefore stretches the $u$ channel more than the $v$ channel, which
explains why the optimizer picks positive shortcut weight while still keeping
both channels positive.

\begin{theorem}[Deterministic ridge optimum]
\label{thm:ridge-invariant}
Under the deterministic family model above, fix $0<\rhoavg<1$ and
$\lambda>0$, and let $\hat w^\lambda=(\hat w_Z^\lambda,\hat w_S^\lambda)$
denote the ridge-logistic minimizer. Then
\[
0 < \hat w_S^\lambda < \hat w_Z^\lambda.
\]
Hence the classifier stays in the invariant cone and has zero $0$-$1$ error on
every family in the deterministic model.
\end{theorem}

\paragraph{Why deterministic attraction is not yet failure.}
The deterministic model still exposes a genuine test-side degradation. For every test
family,
\[
\E_{\test}\bigl[Y g_w(X)\bigr] = w_Z + \rho_{\test} w_S.
\]
If $\rho_{\test}<0$ and $w_S>0$, then the shortcut component strictly lowers
the test margin relative to dropping $S$. Theorem~\ref{thm:ridge-invariant}
matters because it pins down where this degradation stops: ridge-logistic ERM
absorbs positive shortcut weight but remains inside the invariant cone, so the
deterministic scope gives surrogate degradation without forcing
misclassification. In this regime, train-side shortcut attraction and test-side
harm appear without an actual switch to the shortcut rule.

\paragraph{Role of ridge regularization.}
Without regularization, the deterministic model is linearly separable because
$Z=Y$. The negative derivative at $w_S=0$ should therefore be read as a local
shortcut bias. Ridge regularization is used only to convert that local bias
into a finite optimizer statement.

\paragraph{Passing to the noisy regime.}
The deterministic analysis therefore resolves only the first part of the
story: whether positive average shortcut correlation creates pressure toward
the shortcut coordinate. It cannot produce a shortcut-rule transition because
the invariant coordinate is still perfect. To study when train-side shortcut
pressure becomes actual rule selection, we keep the same two-coordinate family
model and relax only the invariant signal by introducing noise level $\gamma$.
This preserves the meaning of $\rhoavg$ as the aggregate training shortcut
strength while making the invariant-shortcut comparison nontrivial.

\paragraph{Noisy-invariant family model.}
Fix $\gamma \in (0,1]$. For each family $e$, let
\[
Y \sim \Unif\{-1,+1\},
\]
let $A,B_e\in\{-1,+1\}$ be independent of $Y$ and independent of each other,
with
\[
\E[A]=\gamma,
\qquad
\E[B_e]=\rho_e,
\]
and define
\[
\; Z=YA,
\qquad
S=YB_e.
\]
If the training distribution mixes observed families
$e_1,\dots,e_m$ with weights $\alpha_1,\dots,\alpha_m$, introduce a family
index
\[
E \sim \sum_{i=1}^m \alpha_i \delta_{e_i},
\]
independent of $(Y,A)$, and define the training-side shortcut agreement
variable
\[
B:=B_E.
\]
Then $\E[B]=\rhoavg$. Equivalently, for each family,
\[
\Prob(Z=Y)=\frac{1+\gamma}{2},
\qquad
\Prob(S=Y)=\frac{1+\rho_e}{2}.
\]

\paragraph{Rule-level bridge in the noisy regime.}
Let the invariant rule be $f_Z(z,s)=z$ and the shortcut rule be $f_S(z,s)=s$.
Under the noisy model,
\[
\risk_{\train}(f_Z)=\frac{1-\gamma}{2},
\qquad
\risk_{\train}(f_S)=\frac{1-\rhoavg}{2},
\]
and on a test family,
\[
\risk_{\test}(f_S)-\risk_{\test}(f_Z)=\frac{\gamma-\rho_{\test}}{2}.
\]
Thus the training distribution already prefers the shortcut rule exactly when
$\rhoavg>\gamma$, which is the training-side shortcut-transition threshold.
A positive test risk gap against the invariant rule appears when the test
family weakens further so that $\rho_{\test}<\gamma$, and sign-flipped test
families with $\rho_{\test}<0$ yield error above chance. The next theorem
shows that full ridge-logistic ERM over linear scores also enters the
shortcut-rule side under the same training inequality, so the transition is
not an artifact of comparing only two hand-picked rules.

The noisy regime also admits a direct coordinate-swap identity. If
$\loss_{\train}^{\mathrm{noisy}}$ denotes the population noisy logistic
objective, then
\[
\loss_{\train}^{\mathrm{noisy}}(w_Z,w_S)
-
\loss_{\train}^{\mathrm{noisy}}(w_S,w_Z)
=
\frac{\rhoavg-\gamma}{2}(w_Z-w_S).
\]
So when $\rhoavg>\gamma$, every point with $w_Z>w_S$ is worse than its
coordinate-swapped version, so the noisy objective pushes the optimizer toward
shortcut-rule behavior. The appendix sharpens this
into an exact sign statement by showing that the sign of
$\hat w_Z^\lambda-\hat w_S^\lambda$ changes at $\rhoavg=\gamma$.

The same threshold has a finite-sample counterpart. If
$\Delta_{\train}:=\rhoavg-\gamma>0$ and empirical risk is minimized over the
two-rule class $\{f_Z,f_S\}$ on $n$ i.i.d.\ draws from the family-averaged
training distribution, then the shortcut rule is selected with probability at
least $1-\exp(-n\Delta_{\train}^2/8)$. We keep that selector-level argument in
the appendix, but mention it here because the finite-sample experiment tracks
the same transition rather than introducing a new story.

\begin{theorem}[Noisy ridge-logistic shortcut transition]
\label{thm:noisy-ridge-shortcut}
Under the noisy-invariant family model above, fix $\lambda>0$ and let
$\hat w^\lambda=(\hat w_Z^\lambda,\hat w_S^\lambda)$ denote the ridge-logistic
minimizer. If $\rhoavg>\gamma$, then
\[
\hat w_Z^\lambda + \hat w_S^\lambda > 0,
\qquad
\hat w_Z^\lambda - \hat w_S^\lambda < 0.
\]
Hence $f_{\hat w^\lambda}(z,s)=s$, and on any test family
\[
\risk_{\test}\!\bigl(f_{\hat w^\lambda}\bigr)
- \risk_{\test}\!\bigl((z,s)\mapsto z\bigr)
=
\frac{\gamma-\rho_{\test}}{2}.
\]
In particular, if $\rho_{\test}<\gamma$, this test risk gap is strictly
positive, and if $\rho_{\test}<0$, then
\[
\risk_{\test}\!\bigl(f_{\hat w^\lambda}\bigr)
=
\frac{1-\rho_{\test}}{2}
>
\frac12.
\]
\end{theorem}

\section{Synthetic Checks}
We report two synthetic checks aligned with the theory: population geometry and
finite-sample noisy ERM. They numerically test the main-text claims rather than
provide benchmark coverage.

\paragraph{Protocol.}
The population panels are computed from closed-form formulas or one-dimensional
root solves. For finite-sample plots we use $15$ sample sizes between $20$ and
$600$, $1400$ repetitions per sample size, balanced family sampling, and
$95\%$ confidence bands across repetitions. Ridge-logistic ERM is optimized
over the four sufficient binary states, so the visible uncertainty comes from
Monte Carlo variation rather than a large-scale optimizer. The anonymous
supplement releases the exact figure-generation script. Regenerating the two
main figures with `--main-only' on a single Apple M4 Max CPU workstation with
36\,GB unified memory took about $4.2$ seconds, peaked at roughly $130$\,MB
memory, and used no GPU or external cluster.

\paragraph{Population geometry.}
Figure~\ref{fig:population-geometry} shows the two analytic regimes. The left
panel reproduces Theorem~\ref{thm:ridge-invariant}: positive average training
shortcut correlation yields positive shortcut weight while the invariant
coefficient stays larger. The right panel plots the sign of $\hat w_S-\hat
w_Z$ over $(\gamma,\rhoavg)$ and recovers the exact sign boundary
$\rhoavg=\gamma$ from the noisy analysis. Shortcut
weight already appears in the deterministic scope, but shortcut-rule behavior
requires the strictly stronger noisy threshold.

\begin{figure}[t]
\centering
\includegraphics[width=\linewidth]{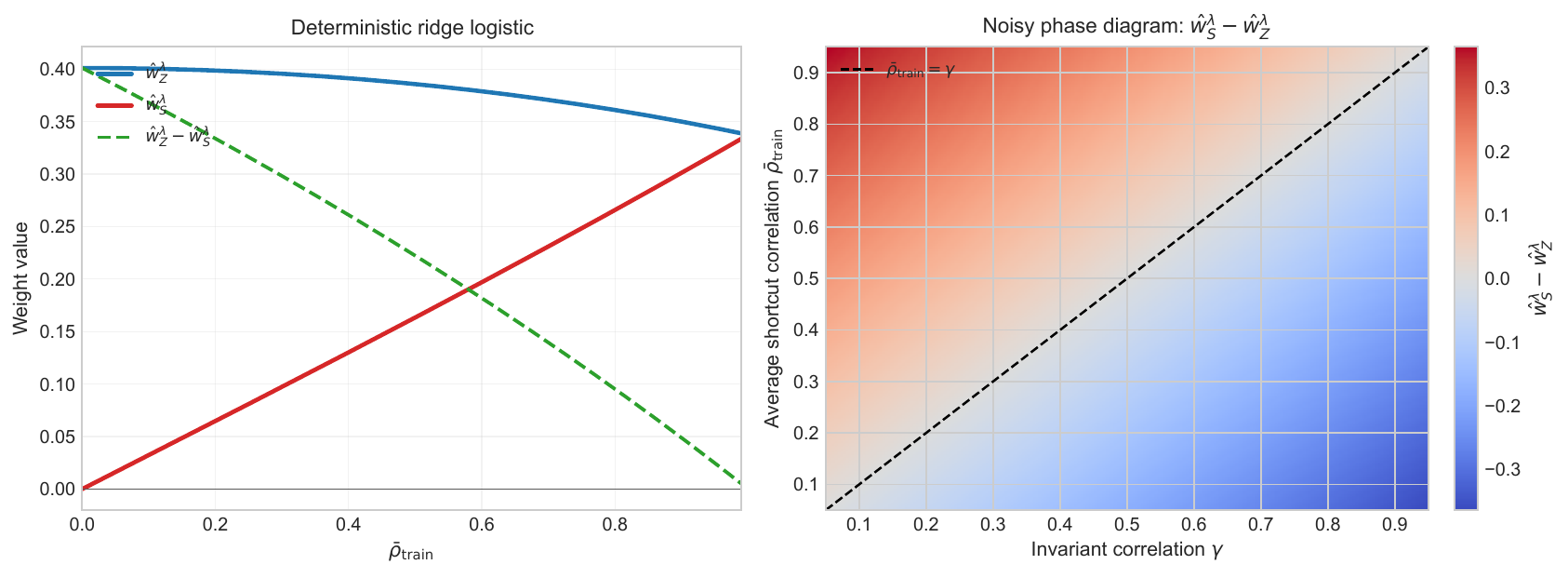}
\caption{Population geometry. Left: deterministic ridge-logistic weights versus
average training shortcut correlation. Right: noisy phase diagram via the sign
of $\hat w_S-\hat w_Z$, with exact sign boundary $\rhoavg=\gamma$.}
\label{fig:population-geometry}
\end{figure}

\paragraph{Finite-sample shortcut selection.}
Figure~\ref{fig:finite-sample-erm} fixes
$(\gamma,\rhoavg)=(0.55,0.80)$ and keeps the training distribution fixed while
changing only the held-out family. The left panel tracks the training-side
probability of shortcut-rule behavior. The right panel evaluates the same
finite-sample estimator on two test families: a failure-side family with
$\rho_{\test}=-0.30$ and a no-failure control with $\rho_{\test}=0.70$. The
ridge-logistic curves therefore illustrate the main separation claim directly:
the training-side transition is the same in both cases, but the held-out
outcome depends on the extra test-side inequality. The sign-flipped family
approaches the absolute-failure side of Theorem~\ref{thm:noisy-ridge-shortcut},
while the control family stays below both chance and the invariant baseline.

\begin{figure}[t]
\centering
\includegraphics[width=\linewidth]{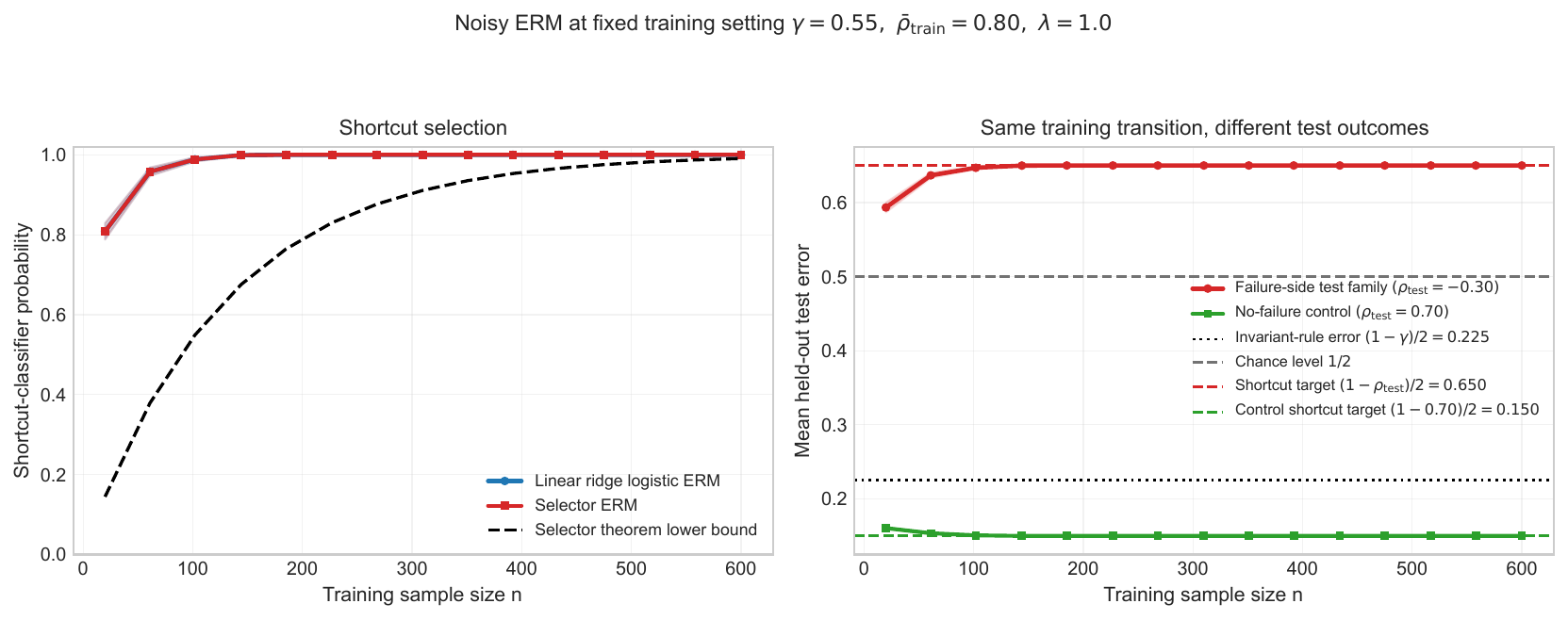}
\caption{Finite-sample noisy ERM at fixed training setting
$(\gamma,\rhoavg)=(0.55,0.80)$. Left: training-side probability of
shortcut-rule behavior. Right: mean test error on a failure-side family
$\rho_{\test}=-0.30$ and on a no-failure control $\rho_{\test}=0.70$,
with invariant-rule and chance baselines.}
\label{fig:finite-sample-erm}
\end{figure}

\paragraph{What the checks show.}
For this paper's scope, Figure~\ref{fig:population-geometry} validates the
deterministic/noisy separation and the training-side sign boundary at
population level, and Figure~\ref{fig:finite-sample-erm} shows that one
shared training transition can yield different held-out outcomes depending on
the test family. We therefore use the experiments as theorem-aligned sanity
checks rather than as a broader empirical study of the full test-family phase
diagram.

\section{Discussion}
The paper studies a narrow diagnostic question. Positive training shortcut
correlation, learned shortcut use, and test-time failure are often discussed
together, but the minimal model shows that they can come apart. Inside one
closed-form parameterization, Theorem~\ref{thm:ridge-invariant} shows that
deterministic shortcut attraction and test-side degradation can occur without
deterministic misclassification, while Theorem~\ref{thm:noisy-ridge-shortcut}
shows what must be added before a shortcut-rule transition and cross-family
failure follow. The appendix selector results reinforce the same control
parameter rather than introducing a separate story.

The asymmetry between the two theorems matters. Theorem~\ref{thm:ridge-invariant}
is not a failure theorem; it rules out an overly strong reading of train-side
evidence. Theorem~\ref{thm:noisy-ridge-shortcut} then adds the stronger noisy
condition that actually moves ridge-logistic ERM onto the shortcut rule. Even
there, the learned training solution is still not enough by itself: failure
only follows after the held-out family satisfies $\rho_{\test}<\gamma$, and
above-chance error appears only on the stricter sign-flip side
$\rho_{\test}<0$.

The experiments follow the same split. Figure~\ref{fig:population-geometry}
checks the population geometry behind both regimes, while
Figure~\ref{fig:finite-sample-erm} keeps the training distribution fixed and
changes only the held-out family, so the same learned shortcut tendency can be
seen producing failure-side and no-failure-side outcomes. The figures verify
the separation logic of the theory; they are not benchmark evidence and do not
map the full test-family phase diagram.

One modeling choice deserves emphasis. The main text parameterizes the observed
training families through the average shortcut strength $\rhoavg$ rather than a
richer family descriptor. In this minimal binary model, the exact population
risk, the rule-level comparison, and the noisy training-side transition all
collapse to the same family-average control parameter, while actual failure
still depends on the separate test-side comparison $\rho_{\test}<\gamma$.
Richer family heterogeneity may matter in larger models, but that is a
different problem from the one studied here.

The broader impact of this work is mainly diagnostic. Positively, clearer
separation between training-side shortcut transition and test-side failure can
help robustness evaluations avoid over-reading training correlations as direct
deployment guarantees. Negatively, this same minimal model could be
over-generalized to settings far outside its binary assumptions or used to
replace empirical auditing with purely analytic arguments, which is why we keep
the claims narrow and treat the figures as theorem-aligned checks rather than
benchmark evidence.

\section{Conclusion}
This paper studies a minimal model of cross-family shortcut-driven OOD failure.
Within one shared binary construction, positive training shortcut correlation
can create shortcut attraction without failure, the stronger noisy inequality
$\rhoavg>\gamma$ can move ridge-logistic ERM onto the shortcut rule, and actual
test-time failure still requires a separate held-out condition
$\rho_{\test}<\gamma$. Sign-flipped families $\rho_{\test}<0$ then yield
above-chance error.

The deterministic theorem shows why train-side shortcut evidence is not enough:
attraction and surrogate degradation already appear before misclassification
does. The noisy theorem adds the threshold for shortcut-rule transition, while
the test family still determines whether that transition remains benign or
becomes actual OOD failure. The finite-sample check mirrors the same split
under a fixed training setup.

The model is binary, the noisy regime uses independent agreement variables for
invariant and shortcut coordinates, and the figures serve as theorem-aligned
checks rather than benchmark evidence. Those
restrictions buy exact formulas and closed-form thresholds. Extending the same
separation to richer family structure is the natural next step.

\bibliographystyle{plainnat}
\bibliography{references}

\appendix
\section{Proofs for the Main Theorems}

Throughout, let
\[
\ell(t)=\log(1+e^{-t}),
\qquad
\sigma(t)=(1+e^{-t})^{-1}.
\]

\begin{proof}[Proof of Theorem~\ref{thm:ridge-invariant}]
Under the deterministic family model, the weighted ridge-logistic objective is
\[
J_\lambda(w_Z,w_S)
:=
\frac{1+\rhoavg}{2}\,\ell(w_Z+w_S)
+
\frac{1-\rhoavg}{2}\,\ell(w_Z-w_S)
+
\frac{\lambda}{2}(w_Z^2+w_S^2).
\]
Set
\[
\alpha:=\frac{1+\rhoavg}{2},
\qquad
\beta:=\frac{1-\rhoavg}{2},
\qquad
u:=w_Z+w_S,
\qquad
v:=w_Z-w_S.
\]
Because $0<\rhoavg<1$, we have $1>\alpha>\beta>0$. Also
\[
w_Z=\frac{u+v}{2},
\qquad
w_S=\frac{u-v}{2},
\qquad
w_Z^2+w_S^2=\frac{u^2+v^2}{2}.
\]
Hence
\[
J_\lambda(w_Z,w_S)
=
\phi_\alpha(u)+\phi_\beta(v),
\qquad
\phi_k(x):=k\ell(x)+\frac{\lambda}{4}x^2.
\]
Each $\phi_k$ is strictly convex, with derivative
\[
\phi_k'(x)=-k\sigma(-x)+\frac{\lambda}{2}x.
\]
Since $\phi_k'(0)=-k/2<0$ and $\phi_k'(x)\to+\infty$ as $x\to+\infty$, the
unique minimizer $x_k^\star$ of $\phi_k$ satisfies $x_k^\star>0$.

Now compare the two channels. For every $x$,
\[
\phi_\alpha'(x)-\phi_\beta'(x)=-(\alpha-\beta)\sigma(-x)<0.
\]
Because both derivatives are strictly increasing and each has exactly one root,
their roots satisfy
\[
x_\alpha^\star>x_\beta^\star>0.
\]
The unique minimizer of $J_\lambda$ is therefore
\[
u^\star=x_\alpha^\star,
\qquad
v^\star=x_\beta^\star,
\]
which implies
\[
\hat w_S^\lambda=\frac{u^\star-v^\star}{2}>0,
\qquad
\hat w_Z^\lambda=\frac{u^\star+v^\star}{2}>\hat w_S^\lambda.
\]
So the minimizer remains in the invariant cone. Because the deterministic model
has $Z=Y$, any score with $w_Z>|w_S|$ predicts $Y$ correctly on every family,
and the resulting $0$-$1$ error is zero.
\end{proof}

\begin{proof}[Proof of Theorem~\ref{thm:noisy-ridge-shortcut}]
Under the noisy family model, define
\[
A:=YZ,
\qquad
B:=YS.
\]
By construction, $A,B\in\{-1,+1\}$ are independent, and under the
family-averaged training distribution induced by the random family index $E$
and the mixture variable $B=B_E$,
\[
\Prob(A=1)=\frac{1+\gamma}{2},
\qquad
\Prob(B=1)=\frac{1+\rhoavg}{2}.
\]
Thus
\begin{align*}
\loss_{\train}^{\mathrm{noisy}}(w_Z,w_S)
&:=\E_{\train}[\ell(Y(w_ZZ+w_SS))] \\
&=
\frac{(1+\gamma)(1+\rhoavg)}{4}\,\ell(w_Z+w_S)
+
\frac{(1+\gamma)(1-\rhoavg)}{4}\,\ell(w_Z-w_S) \\
&\qquad
+
\frac{(1-\gamma)(1+\rhoavg)}{4}\,\ell(-w_Z+w_S)
+
\frac{(1-\gamma)(1-\rhoavg)}{4}\,\ell(-w_Z-w_S).
\end{align*}
Set again
\[
u:=w_Z+w_S,
\qquad
v:=w_Z-w_S.
\]
Using $\ell(-t)=\ell(t)+t$, the objective separates as
\[
\loss_{\train}^{\mathrm{noisy}}(w_Z,w_S)
=
\phi_u(u)+\phi_v(v),
\]
where
\[
\phi_u(u)
:=
\frac{1+\gamma\rhoavg}{2}\,\ell(u)
+
\frac{(1-\gamma)(1-\rhoavg)}{4}\,u,
\]
\[
\phi_v(v)
:=
\frac{1-\gamma\rhoavg}{2}\,\ell(v)
+
\frac{(1-\gamma)(1+\rhoavg)}{4}\,v.
\]
Adding the ridge penalty gives
\[
J_\lambda^{\mathrm{noisy}}(w_Z,w_S)
=
\Phi_u(u)+\Phi_v(v),
\]
with
\[
\Phi_u(u):=\phi_u(u)+\frac{\lambda}{4}u^2,
\qquad
\Phi_v(v):=\phi_v(v)+\frac{\lambda}{4}v^2.
\]
Each term is strictly convex, so the minimizer is unique and separates across
$u$ and $v$.

Differentiate:
\[
\Phi_u'(u)
=
-\frac{1+\gamma\rhoavg}{2}\sigma(-u)
+
\frac{(1-\gamma)(1-\rhoavg)}{4}
+
\frac{\lambda}{2}u,
\]
\[
\Phi_v'(v)
=
-\frac{1-\gamma\rhoavg}{2}\sigma(-v)
+
\frac{(1-\gamma)(1+\rhoavg)}{4}
+
\frac{\lambda}{2}v.
\]
At the origin,
\[
\Phi_u'(0)
=
-\frac{1+\gamma\rhoavg}{4}
+
\frac{(1-\gamma)(1-\rhoavg)}{4}
=
-\frac{\gamma+\rhoavg}{4}
<
0,
\]
and, because $\rhoavg>\gamma$,
\[
\Phi_v'(0)
=
-\frac{1-\gamma\rhoavg}{4}
+
\frac{(1-\gamma)(1+\rhoavg)}{4}
=
\frac{\rhoavg-\gamma}{4}
>
0.
\]
Since $\Phi_u'$ is strictly increasing and tends to $+\infty$ as
$u\to+\infty$, its unique root satisfies $u^\star>0$. Since $\Phi_v'$ is
strictly increasing and tends to $-\infty$ as $v\to-\infty$, its unique root
satisfies $v^\star<0$. Therefore
\[
\hat w_Z^\lambda+\hat w_S^\lambda=u^\star>0,
\qquad
\hat w_Z^\lambda-\hat w_S^\lambda=v^\star<0.
\]

To identify the induced classifier, fix any $(z,s)\in\{-1,+1\}^2$. If $z=s$,
then
\[
\hat w_Z^\lambda z+\hat w_S^\lambda s=u^\star s,
\]
whose sign is $s$ because $u^\star>0$. If $z=-s$, then
\[
\hat w_Z^\lambda z+\hat w_S^\lambda s=v^\star z=-v^\star s,
\]
whose sign is again $s$ because $v^\star<0$. Hence
\[
f_{\hat w^\lambda}(z,s)=s
\qquad
\text{for every } (z,s)\in\{-1,+1\}^2.
\]

On a test family, the shortcut rule therefore has risk
\[
\risk_{\test}(f_{\hat w^\lambda})
=
\Prob_{\test}(S\neq Y)
=
\frac{1-\rho_{\test}}{2},
\]
while the invariant rule has risk
\[
\risk_{\test}\bigl((z,s)\mapsto z\bigr)
=
\Prob_{\test}(Z\neq Y)
=
\frac{1-\gamma}{2}.
\]
Subtracting gives
\[
\risk_{\test}(f_{\hat w^\lambda})
-\risk_{\test}\bigl((z,s)\mapsto z\bigr)
=
\frac{\gamma-\rho_{\test}}{2}.
\]
If $\rho_{\test}<0$, then also
\[
\risk_{\test}(f_{\hat w^\lambda})
=
\frac{1-\rho_{\test}}{2}
>
\frac12,
\]
so the learned shortcut rule fails absolutely on the shifted test family.
\end{proof}

\paragraph{Exact sign boundary in the noisy regime.}
The same proof yields more than the failure-side statement in
Theorem~\ref{thm:noisy-ridge-shortcut}. Since
\[
\Phi_v'(0)=\frac{\rhoavg-\gamma}{4},
\]
strict convexity implies that the unique minimizer $v^\star=\hat w_Z^\lambda-\hat
w_S^\lambda$ satisfies
\[
v^\star
\begin{cases}
>0, & \rhoavg<\gamma,\\[2pt]
=0, & \rhoavg=\gamma,\\[2pt]
<0, & \rhoavg>\gamma.
\end{cases}
\]
So the sign of $\hat w_Z^\lambda-\hat w_S^\lambda$ changes exactly at
$\rhoavg=\gamma$, which is the boundary visualized in
Figure~\ref{fig:population-geometry}.

\section{Supplementary Selector Results}

Let
\[
\mathcal{F}_{\mathrm{sel}}=\{f_Z,f_S\},
\qquad
f_Z(z,s)=z,
\qquad
f_S(z,s)=s.
\]

\paragraph{Population selector ERM.}
Under the noisy-invariant family model,
\[
\risk_{\train}(f_Z)=\frac{1-\gamma}{2},
\qquad
\risk_{\train}(f_S)=\frac{1-\rhoavg}{2}.
\]
Hence if $\rhoavg>\gamma$, population ERM over $\mathcal{F}_{\mathrm{sel}}$
chooses $f_S$. On a test family,
\[
\risk_{\test}(f_S)-\risk_{\test}(f_Z)
=
\frac{1-\rho_{\test}}{2}-\frac{1-\gamma}{2}
=
\frac{\gamma-\rho_{\test}}{2}.
\]
So whenever $\rho_{\test}<\gamma$, this selector-level shortcut choice incurs
strictly positive OOD risk gap over the invariant rule.

\paragraph{Finite-sample selector ERM.}
Let $\Delta_{\train}:=\rhoavg-\gamma>0$, and let
\[
\hat f_n \in \argmin_{f\in\mathcal{F}_{\mathrm{sel}}}\widehat{\risk}_n(f)
\]
minimize empirical risk over $n$ i.i.d.\ draws from the family-averaged
training distribution. Define
\[
W_i:=\one\{f_S(X_i)\neq Y_i\}-\one\{f_Z(X_i)\neq Y_i\}.
\]
Then $W_i\in[-1,1]$ and
\[
\widehat{\risk}_n(f_S)-\widehat{\risk}_n(f_Z)
=
\frac{1}{n}\sum_{i=1}^n W_i.
\]
Moreover,
\[
\E[W_i]
=
\risk_{\train}(f_S)-\risk_{\train}(f_Z)
=
\frac{\gamma-\rhoavg}{2}
=
-\frac{\Delta_{\train}}{2}.
\]
Applying Hoeffding's inequality,
\[
\Prob\!\bigl(\widehat{\risk}_n(f_S)\ge\widehat{\risk}_n(f_Z)\bigr)
\le
\exp\!\left(-\frac{n\Delta_{\train}^2}{8}\right),
\]
so
\[
\Prob\!\bigl(\widehat{\risk}_n(f_S)<\widehat{\risk}_n(f_Z)\bigr)
\ge
1-\exp\!\left(-\frac{n\Delta_{\train}^2}{8}\right).
\]
On this event, $\hat f_n=f_S$, and therefore for any test family,
\[
\risk_{\test}(\hat f_n)-\risk_{\test}(f_Z)
=
\frac{\gamma-\rho_{\test}}{2}.
\]
This is the selector-level finite-sample concentration bound referenced in the
main text.

\section*{NeurIPS Paper Checklist}

\begin{enumerate}

\item {\bf Claims}
    \item[] Question: Do the main claims made in the abstract and introduction accurately reflect the paper's contributions and scope?
    \item[] Answer: \answerYes{}
    \item[] Justification: The abstract and Introduction explicitly state the minimal-model scope, the deterministic/noisy split, and the extra test-side conditions needed for OOD failure; these claims match Sections~1--5 and the appendix proofs.

\item {\bf Limitations}
    \item[] Question: Does the paper discuss the limitations of the work performed by the authors?
    \item[] Answer: \answerYes{}
    \item[] Justification: Sections~1, 4, and~5 discuss the narrow binary setting, the independent agreement-variable assumptions, and the fact that the experiments are theorem-aligned synthetic checks rather than benchmark evidence.

\item {\bf Theory assumptions and proofs}
    \item[] Question: For each theoretical result, does the paper provide the full set of assumptions and a complete (and correct) proof?
    \item[] Answer: \answerYes{}
    \item[] Justification: Section~2 states the deterministic and noisy family models and theorem assumptions, and Appendix~A provides full proofs for the two main theorems together with the supplementary selector results.

    \item {\bf Experimental result reproducibility}
    \item[] Question: Does the paper fully disclose all the information needed to reproduce the main experimental results of the paper to the extent that it affects the main claims and/or conclusions of the paper (regardless of whether the code and data are provided or not)?
    \item[] Answer: \answerYes{}
    \item[] Justification: Section~3 gives the exact parameter settings, sample-size grid, repetitions, balanced family sampling, confidence-band definition, and optimization setup, and the anonymous supplementary material contains the exact script used to regenerate the two figures.

\item {\bf Open access to data and code}
    \item[] Question: Does the paper provide open access to the data and code, with sufficient instructions to faithfully reproduce the main experimental results, as described in supplemental material?
    \item[] Answer: \answerYes{}
    \item[] Justification: The experiments are synthetic and require no external data. The anonymous supplementary zip includes the figure-generation script, the generated PDFs, and a minimal command to reproduce the main figures.

\item {\bf Experimental setting/details}
    \item[] Question: Does the paper specify all the training and test details (e.g., data splits, hyperparameters, how they were chosen, type of optimizer) necessary to understand the results?
    \item[] Answer: \answerYes{}
    \item[] Justification: Section~3 specifies the train/test family parameters, sample-size range, number of repetitions, balanced family sampling, confidence bands, and ridge-logistic optimization over the four sufficient binary states; the supplementary script fixes the remaining implementation details.

\item {\bf Experiment statistical significance}
    \item[] Question: Does the paper report error bars suitably and correctly defined or other appropriate information about the statistical significance of the experiments?
    \item[] Answer: \answerYes{}
    \item[] Justification: Section~3 reports $95\%$ confidence bands across $1400$ repetitions per sample size and states that the plotted variability comes from Monte Carlo variation across repetitions.

\item {\bf Experiments compute resources}
    \item[] Question: For each experiment, does the paper provide sufficient information on the computer resources (type of compute workers, memory, time of execution) needed to reproduce the experiments?
    \item[] Answer: \answerYes{}
    \item[] Justification: Section~3 states that regenerating the two main figures used a single Apple M4 Max CPU workstation with 36\,GB unified memory, took about $4.2$ seconds, peaked at roughly $130$\,MB memory, and used no GPU or external cluster.
    
\item {\bf Code of ethics}
    \item[] Question: Does the research conducted in the paper conform, in every respect, with the NeurIPS Code of Ethics \url{https://neurips.cc/public/EthicsGuidelines}?
    \item[] Answer: \answerYes{}
    \item[] Justification: The work is synthetic theoretical/experimental research with no human subjects, scraped data, or deployment-facing release, and the supplementary material is limited to a small anonymous reproduction package.

\item {\bf Broader impacts}
    \item[] Question: Does the paper discuss both potential positive societal impacts and negative societal impacts of the work performed?
    \item[] Answer: \answerYes{}
    \item[] Justification: Section~4 discusses both positive impacts (clearer robustness evaluation under shift) and negative impacts (over-generalizing a minimal model or replacing empirical auditing with purely analytic arguments).
    
\item {\bf Safeguards}
    \item[] Question: Does the paper describe safeguards that have been put in place for responsible release of data or models that have a high risk for misuse (e.g., pre-trained language models, image generators, or scraped datasets)?
    \item[] Answer: \answerNA{}
    \item[] Justification: The paper does not release high-risk models, datasets, or scraped corpora; the supplementary material only contains a small synthetic figure-generation script and the corresponding PDFs.

\item {\bf Licenses for existing assets}
    \item[] Question: Are the creators or original owners of assets (e.g., code, data, models), used in the paper, properly credited and are the license and terms of use explicitly mentioned and properly respected?
    \item[] Answer: \answerNA{}
    \item[] Justification: The main claims do not rely on external datasets, pretrained models, or third-party research assets; the experiments are synthetic and self-contained, and standard numerical libraries are only software dependencies.

\item {\bf New assets}
    \item[] Question: Are new assets introduced in the paper well documented and is the documentation provided alongside the assets?
    \item[] Answer: \answerYes{}
    \item[] Justification: The anonymous supplementary zip includes a README with the reproduction command and outputs, together with the exact script and generated figure files for the released synthetic assets.

\item {\bf Crowdsourcing and research with human subjects}
    \item[] Question: For crowdsourcing experiments and research with human subjects, does the paper include the full text of instructions given to participants and screenshots, if applicable, as well as details about compensation (if any)? 
    \item[] Answer: \answerNA{}
    \item[] Justification: The paper does not involve crowdsourcing or research with human subjects.

\item {\bf Institutional review board (IRB) approvals or equivalent for research with human subjects}
    \item[] Question: Does the paper describe potential risks incurred by study participants, whether such risks were disclosed to the subjects, and whether Institutional Review Board (IRB) approvals (or an equivalent approval/review based on the requirements of your country or institution) were obtained?
    \item[] Answer: \answerNA{}
    \item[] Justification: The paper does not involve crowdsourcing or research with human subjects.

\item {\bf Declaration of LLM usage}
    \item[] Question: Does the paper describe the usage of LLMs if it is an important, original, or non-standard component of the core methods in this research? Note that if the LLM is used only for writing, editing, or formatting purposes and does \emph{not} impact the core methodology, scientific rigor, or originality of the research, declaration is not required.
    \item[] Answer: \answerNA{}
    \item[] Justification: LLMs are not part of the core methodology, proofs, or experiments in this research.

\end{enumerate}

\end{document}